\documentclass[journal]{IEEEtran}
\usepackage{color}

\usepackage{cite}

\usepackage{amsmath}
\usepackage{amsfonts}
\usepackage{amssymb}
\usepackage{amsthm}
\newtheorem{theorem}{Theorem}
\newtheorem{corollary}{Corollary}[theorem] 
\newtheorem{lemma}[theorem]{Lemma}
\theoremstyle{definition}

\theoremstyle{remark}

\ifCLASSINFOpdf
  \usepackage[pdftex]{graphicx}
  \DeclareGraphicsExtensions{.pdf,.jpeg,.jpg}
\else
\fi
\usepackage{url}

\begin{document}
%
\title{3D traffic flow model for UAVs}
%
%
%

\author{Mirmojtaba~Gharibi, Raouf~Boutaba,~\IEEEmembership{Fellow,~IEEE,}
        and~Steven~L.~Waslander,~\IEEEmembership{Senior Member,~IEEE}
\thanks{}
\thanks{M. Gharibi and R. Boutaba are with D. Cheriton School
of Computer Science, University of Waterloo, Waterloo,
ON, N2L 3G1, Canada. Contacts: mgharibi@uwaterloo.ca and rboutaba@uwaterloo.ca respectively.}
\thanks{S. L. Waslander S. L. Waslander is with the Institute for Aerospace Studies, University of
Toronto, Toronto, ON M3H 5T6, Canada, stevenw@utias.utoronto.ca}
}

%
%

\markboth{Gharibi \MakeLowercase{\textit{et al.}}: 3D traffic flow model for UAVs }%
{Gharibi \MakeLowercase{\textit{et al.}}: 3D traffic flow model for UAVs }
%



\maketitle

\begin{abstract}
In this work, we introduce a microscopic traffic flow model called Scalar Capacity Model (SCM) which can be used to study the formation of traffic on an airway link for autonomous Unmanned Aerial Vehicles (UAV) as well as for the ground vehicles on the road. Given the 3D nature of UAV flights, the main novelty in our model is to eliminate the commonly used notion of lanes and replace it with a notion of density and capacity of flow, but in such a way that individual vehicle motions can still be modeled. We name this a Density/Capacity View (DCV) of the link capacity and how vehicles utilize it versus the traditional One/Multi-Lane View (OMV). An interesting feature of this model is exhibiting both passing and blocking regimes (analogous to multi-lane or single-lane) depending on the set scalar parameter for capacity. We show the model has linear local (platoon) and string stability. Also, we perform numerical simulations and show evidence for non-linear stability. Our traffic flow model is represented by a nonlinear differential equation which we transform into a linear form. This makes our model analytically solvable in the blocking regime and piece-wise analytically solvable in the passing regime.
\end{abstract}


%
\IEEEpeerreviewmaketitle

\section{Introduction}\label{sec:introduction}
 

%
%
%
%

\IEEEPARstart{T}{he Unmanned Aerial Vehicles (UAV) } will soon be common place. They will do a variety of tasks such as on-demand aerial package delivery, search and rescue operations, agriculture, cinematography, inspection of infrastructure, and wild life and traffic surveillance \cite{Gha16}. However, this is a field that is still in its infancy and main ideas for integration of UAVs in the airspace are just starting to appear \cite{Gha16,Kou17,Waz18,Kha18,UTM15,Kop15,UTM_fact15}. To enable such a reality, various technical tools are needed, including traffic flow models over a single link to study the formation of congestion in the air.

The goal of traffic flow research is to study the interaction between the vehicles and the transportation network and design efficient transportation networks from the learned insights. These insights are often conceptualized via mathematical modelling. In their traditional domain of ground vehicles, traffic flow models help with understanding the formation of traffic jams as a result of various flow conditions, driving behaviors, road structures such as on-ramps and off-ramps, etc. They will play an analogous role for UAVs. 

Developing microscopic traffic flow modeling for UAVs is a new problem with its unique set of requirements. The closest related research area we can look for solutions is that of traffic flow models for ground vehicles. As we will see, even the limited existing works on UAV traffic flow models are adaptations of ground vehicle traffic flow models. A main characteristic of traffic flow models for ground vehicles is that they structure the road into one or multi lanes and allow the movement of vehicles in this 2D space \cite{Jia01}. We call this general view of the modelling  One/Multi-lane View (OMV). Within OMV, in the simpler case of one lane, no passing occurs. Most models are first introduced as one lane models  and then with the aid of a separate lane changing model are extended to multi lane models \cite{Jia01, Kes07,Tol03}. 

An OMV-based model is limited in its application to UAVs as their movements are in the 3D space and lanes are not defined. Furthermore, not only the pass planning aspect is ambiguous in the 3D space, but also a low level detail that adds to the complexity of a microscopic model and therefore should be aggregated. This is so since the overall goal is understanding the longitudinal movements of vehicles along the highway. Finally, in OMV models, a velocity will be assigned to each vehicle based on the congestion in their lane. In the same vein, it is ambiguous how the velocity must be determined in the 3D space with no lanes.

The main problem is to formulate a traffic flow model in a 3D space with no lanes for UAVs. We solve this problem by using a concept of a channel in which vehicles move and a density/capacity framework where for a vehicle to move forward, the density (or congestion) in its horizon must be under the set capacity of the channel. That is the velocity of each vehicle is set based on the perceived congestion. We call this general view in modelling, a Density/Capacity View (DCV) as an alternative to OMV. A DCV-based model also aggregates the pass planning aspect by allowing a vehicle to pass when the congestion is sufficiently low.

In this work, the main novelty is to eliminate lanes and formulate a DCV-based microscopic traffic flow model for UAVs with application to  ground vehicles as well. Furthermore, our model can exhibit both blocking and passing regimes (analogous to one and multi-lane models) by setting a scalar capacity parameter $\kappa$ below or above a threshold, respectively. 
Our model is among a few models \cite{Has99,New61,Whi90} that can be solved analytically in the blocking regime and piece-wise analytically in the passing regime.
In contrast to the existing literature on multi-anticipation \cite{Tre13,Tre06,Len99,Eis03,Gip81}, our model sets the velocity for each vehicle in a novel way by calculating the overall density in  front of each vehicle and imposing a decaying exponential weight on the distances to every vehicle in the front. Finally, we prove various properties for our proposed model, including stability analysis for the blocking case and the characterization of the asymptotic behavior in the passing case.

\section{Related works}\label{sec relevant work}

Car following theories model the vehicles' movements on a single lane as they follow each other \cite{Jia01}. There are separate lane changing models such as MOBIL (short for Minimizing Overall Braking Induced by Lane change) \cite{Kes07} or the model in \cite{Tol03} that are used to extend these models to multilanes.

Most (if not all) the modern microscopic models are modelled as either single lane or multilane. These include most of the well-known traffic flow models (and their extensions) such as Optimal Velocity Model (OVM)\cite{Ban95}, Full Velocity Difference Model (FVDM)\cite{Jia01}, Intelligent Driver Model (IDM)\cite{Tre00}, and Newell's Car-Following Model\cite{New02}.

We argued in the introduction that pass planning should be aggregated. It is worth noting that in \cite{Lav06}, for macroscopic models (with lanes), authors define a rate of lane changing based on macroscopic quantities such as density. In \cite{Lav08}, based on the work of \cite{Lav06}, authors combine this with a microscopic model together with quantizing the prescribed rate to make it applicable to the microscopic model. However, still the model is essentially OMV-based, although to some extent the lane changing modeling complexity is avoided.

\subsubsection*{UAV traffic flow models}
The literature in this area is very sparse. We are aware of the following two studies.

To integrate UAVs in the airspace, researchers in NASA \cite{Jan17}, propose various structures for the airspace; including a road network like design (below the skyline; that is the tallest building height in a city) similar to our work in \cite{Gha16}. They set certain behavioral rules (i.e. a traffic flow model) for UAVs and accordingly extract the fundamental diagram of flow versus density. However, no stability analysis is done which is the standard in the traffic engineering community. Authors  perform only a numerical simulation under an acceleration from a standstill, followed by cruising and then braking of the leader on a flight lane. The traffic flow model is an OMV-based 1-lane model similar to that of ground vehicle models. In the model, authors consider the reaction delay. Their traffic flow model is based on a constant gain controller that adjusts the velocity to reach a goal velocity for some required separation. Also, the lane change is done collaboratively utilizing wireless communication between vehicles.

In \cite{Bat17}, with the goal of studying the wind effect on the fundamental diagram, the authors extend a car following model by Greenshields et al. \cite{Gre35} to include the wind force. This is a 1-lane model and no stability analysis is performed for the new model beyond what is already done for the original model by the research community. 

\subsubsection*{Ground vehicle traffic flow models}\label{subsec ground vehicle related work}
Traffic flow theory finds its root in the work of Greenshields in 1930s \cite{Gre35}. Traffic flow models can be classified across different dimensions, such as the aggregation level. Macroscopic models take a high level view of traffic flow similar to the flow of liquids or gases. Quantities of interest are local density, flow, mean speed and variance and their evolution through time \cite{Lig55,Ric56,Tre13, Kha12, Gni11,Kum14, Li04}. Microscopic models (e.g. see below) to which our models belong such as car-following or cellular automata models describe the interaction of each driver with its environment. In these models, we are interested in quantities such as individual position and speed and perhaps acceleration\cite{Tre13}.

Within microscopic models, we categorize the models based on their relevance to our model. In particular, a distinction is made between 1-lane or multi-lane models. Many of the classic models are 1-lane models. Among the classics are the Optimal Velocity Model (OVM)\cite{Ban95}, Full Velocity Difference Model (FVDM)\cite{Jia01}, and Intelligent Driver Model (IDM)\cite{Tre00} whereas \cite{Li19} is a more recent example. However, it is possible to extend these to multi-lane models by use of a lane change model such as MOBIL which dictates the rule of when it is safe and beneficial for a vehicle to change lanes\cite{Kes07}. 

Another distinction is whether a vehicle takes the optimal velocity in equilibrium instantly similar to our model or gradually. Models with delays are able to demonstrate delay-induced traffic phenomena at the expense of added complexity. No delay classic models include Reuschel and Pipe's models \cite{Reu50,Pip53}. Classic models such as OVM\cite{Ban95}, FVDM\cite{Jia01}, and Newell's Car-Following Model\cite{New02} exhibit delay.

Another distinction is whether the drivers only react to the immediate vehicle in the front or beyond. In particular, in multi-vehicle anticipation models, a few vehicles at the front are considered by the driver for better stability (fewer accidents) \cite{Tre13}. In \cite{Tre06}, the authors extend some of the traffic flow models including OVM, FVDM, and IDM by adding multi-vehicle anticipation features. In \cite{Len99} and \cite{Eis03}, the authors extend OVM and Gipps\cite{Gip81}.

Another distinction is whether the velocity is adjusted based on the time gaps between two vehicles or the space gaps (such as our model). Models such as FVDM \cite{Jia01} and \cite{Tre00} use time gaps whereas OVM \cite{Ban95} and Newell's car following model \cite{New02} use space gaps. 

We know of very few models that can be solved analytically. A 1-lane model by Hasebe et al. \cite{Has99} uses the tangent hyperbolic function to relate the distance between only subsequent vehicles to their velocity with exact solution for various delays. 

In a highly related work \cite{New61}, Newell designs a 1-lane model that can be solved analytically. It was later extended by Whitham\cite{Whi90}, finding  various exact wave solutions, such as periodic and solitary waves. The model assigns the velocity at time $t+\Delta$ to a follower vehicle according to an exponential decay congestion term at time $t$ where $\Delta$ is a delay constant. The congestion term is based on only the distance between the follower and the leader. This results in a non-linear differential equation which Newell transforms into a linear form when $\Delta=0$ and cars are identical. There are similarities and differences in how this model relates to our work. We used a similar technique to make our differential equations linear. Also, we use an exponential decay scheme, but our formulation is different in that we use all the vehicles in the front and not just the first one. Our model is DCV-based and can exhibit passing or blocking behavior according to the set value for capacity whereas this is a 1-lane model. Furthermore, except of having the same horizon for each car, we do not require cars to be identical. Certain details of the models are also different. For example, our model being DCV based, does not have a concept of minimum headway or vehicle length.

Stability analysis is an important part of the study of any traffic flow model. References \cite{Tre13} and \cite{Wil11}, establish various needed stability criteria for a traffic flow model.

 
 


\section{Model}
In our model, we consider a sequence of vehicles numbered as $0$ up to $N-1$ from the first to the last vehicle travelling along an infinite link. The position of each vehicle is designated by $x_i$ with respect to some chosen origin (Fig. \ref{fig:Fig1}). 

\begin{figure}[!t]
\centering
\includegraphics[trim=0 360 160 0, clip,width=3.5in]{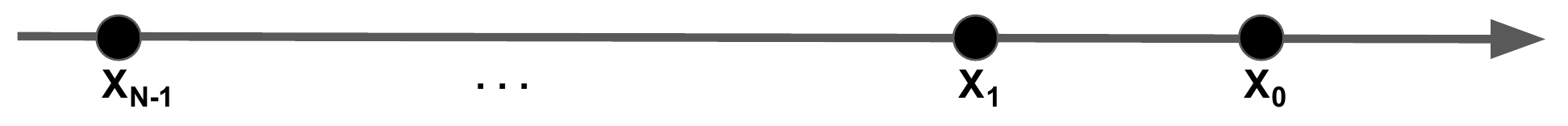}
\caption{Vehicle $i$'s position on the one directional link is shown with $x_i$. The first vehicles is indexed $0$.}
\label{fig:Fig1}
\end{figure}

The vehicles adjust their speeds based on the distances to the vehicles in front of them according to some exponential weighting scheme. This model is represented by the non-linear differential equations described by Eq. \ref{eq-model1}, \ref{eq-model2}, and \ref{eq-model3} as follows. 
\begin{equation}\label{eq-model1}
\frac{dx_i}{dt}=V_i\left(1-\Gamma_i\right)
\end{equation}
\begin{equation}\label{eq-model2}
\Gamma_i = \frac{1}{\kappa}\sum_{0\leq j<i}\exp\left(\frac{x_i-x_j}{\omega}\right)
\end{equation}
\begin{equation}\label{eq-model3}
\frac{dx_0}{dt}=V_0
\end{equation}
where the constant $V_i$ is the maximum free flow speed for vehicle $i$ and $\Gamma_i$ is the congestion factor. The constant $\omega$ is called the horizon in front of each vehicle. Once the leading cars are inside this horizon, they will have a substantial effect on slowing down vehicle $i$, otherwise their effects will be small. Parameter $\kappa$ is called capacity. Intuitively, $\kappa$ is roughly the maximum number of vehicles permitted inside the horizon $\omega$. One way to see this is that if all vehicles in front of a vehicle $i$ are located right in front of it, it takes $\kappa$ vehicles for $\Gamma_i$ to be $1$ in Eq. \ref{eq-model2} and as a result vehicle $i$ to slow down to $0$ velocity (i.e. a perfect jam). However, it is worth mentioning that $\kappa$ need not be an integer and can take any real positive value.

In accordance with this, in section \ref{sec-passingCase}, we prove that given $0<\kappa\leq1$, faster vehicles cannot overtake slower vehicles, corresponding to effectively a 1-lane link (since intuitively only $1$ vehicle is allowed in $i$'th vehicle's horizon as explained above). However, if $1<\kappa$, faster vehicles might be able to pass slower vehicles if certain conditions are satisfied; corresponding to a multi-lane link. We refer to these two different regimes as \textit{passing} and \textit{blocking} hereafter.

\subsection{Discussion and design philosophy}
In this section, we discuss the model in greater depth as well as some of the details of the model. 
Some of the main design decisions or features of the model are already presented in the introduction and throughout the paper and we do not revisit them here.

Originally, in our architecture, Internet of Drones (IoD) \cite{Gha16}, we proposed each airway to be a single lane to reduce technological burden on drones to safely execute a passing maneuver. However, it is plausible that as technology matures, allowing passing will increase the efficiency of airway usage. We are interested in both of these cases in this paper.

We argued earlier that the pass planning should be aggregated. In pass planning, we are dealing with specific maneuvers that happen for a vehicle to change its lateral position (in DCV models) or lane (in OMV models) which has a low relevance to the goal of studying the longitudinal movements. Furthermore, from a technical perspective, passing maneuvers for UAVs is less structured and require a more complex passing model.

One difference between the ground vehicles and autonomous UAVs is the delay aspect. We have assumed the delays for an autonomous vehicle to adjust its velocity according to the traffic condition is negligible. This is not  an entirely correct assumption as while it is plausible to assume the perception and reaction time will be very small compared to the human operated vehicles, still there will be a delay component dictated by the mechanical properties of the system and its inertia. 

Another design choice that we made was the use of space gaps between vehicles compared to the time gaps. Time gaps seem to be the reasonable choices in cases where there is a high disparity between the maximum velocities of different vehicles. But they also lack a crucial component for use for the airway. Since it is expected that the airway links will be very low altitude, they will be affected by the wind disturbances present in the urban centers. These can displace a UAV by several meters. Therefore, it seems the safest choice is to  space vehicles apart enough to safeguard for these disturbances. While time gaps are important as well, we cannot rely solely on them to ensure the safety of flights.

A difference between our model and multi-anticipation models as reviewed in section \ref{sec relevant work} is how the congestion is calculated. In our model in accordance to DCV, we take into account all the vehicles at the front whereas in multi-anticipation models, given the OMV frameworks, only the vehicles on the same lane are considered.

Our model makes it easy to introduce stationary or moving bottlenecks without modifying the model. For example, in the DCV framework, we can adjust the capacity locally by adding dummy vehicles (stationary or moving) whereas in the OMV case, we need to deal with explicit lane closures.


\section{Analytical solution}
We study the passing and blocking regimes separately below.
\subsection{blocking regime}
We use a differential equation technique to transform the characterizing differential equation (Eq. \ref{eq-model1}) into a linear  differential equation. A similar technique was used in \cite{New61}.
Defining the auxiliary variable 
\begin{equation}\label{eq def z_i}
z_i:=\exp\left(\frac{-x_i}{\omega}\right)
\end{equation}
we will have 
\begin{equation}\label{EQ-variableChange1}
\frac{dx_i}{dt}=\frac{-\omega}{z_i}\cdot\frac{dz_i}{dt}.
\end{equation}
Replacing $z_i$ in Eq. \ref{eq-model1} and Eq. \ref{eq-model2} will yield
\begin{equation}\label{EQ-variableChange2}
\frac{-\omega}{z_i}\cdot\frac{dz_i}{dt} = V_i\left(1-\frac{1}{\kappa}\sum_{0\leq j<i}\frac{z_j}{z_i}\right).
\end{equation}
After simplifications, we will have
\begin{equation}\label{EQ-variableChange3}
\frac{dz_i}{dt} = \frac{-V_i}{\omega}z_i+\frac{V_i}{\kappa\omega}\sum_{0\leq j<i}z_j.
\end{equation}
Eq. \ref{EQ-variableChange3} creates a set of homogeneous linear differential equations. There is no shortage of ways to solve this set of equations. One particular way which is especially applicable here is to solve a series of first order linear differential equations as follows. First let us define
\begin{equation}
    Z_{i}(t)=\frac{V_i}{\kappa\omega}\sum_{0\leq j<i}z_j.
\end{equation}
Starting from $z_1$, it can be solved by solving the following differential equation 
\begin{equation}
    \frac{dz_1}{dt} = \frac{-V_1}{\omega}z_1+Z_1(t).
\end{equation}
Since $Z_1$ is a known function in time, $z_1$ can be solved easily by the standard methods as it is a first order linear differential equation. As a result, now $Z_2$ is a known function in time, and similarly $z_2$ can be solved. Applying this method recursively, the whole set of equations can be solved by solving the resulting first order linear equation for each $z_i$.

By solving the set of equations using a method like above, in the simple case where all $V_i$'s are unique, the general solution to this set of differential equations can be written as
\begin{equation}
z_i(t) = \sum_{0\leq j\leq i}c_{i,j}\exp\left(\frac{-V_jt}{\omega}\right)
\end{equation}
where $c_{i,j}$ will be determined using the initial conditions.

In the case where velocities are not unique, the solution looks a bit more involved, but can be expressed in the following way. First let $U$ be the set of smallest indices of vehicles with unique maximum velocities. Let $m_{i,j}$ be the multiplicity of each velocity $V_j$ for vehicles $0$ to $i$ (that is those ahead of vehicle $i$). 
Then the solution for $z_i$ will be of form  
\begin{equation}\label{eq z_i solution}
z_i = \sum_{\substack{j\in U \\ j\leq i}}\sum_{0\leq d<m_{i,j}}c_{i,j,d}\cdot t^d\exp\left(\frac{-V_jt}{\omega}\right)
\end{equation}
and $c_{i,j,d}$ will be determined by the initial conditions.
 
\subsection{Passing regime}
The same analytical approach of the blocking regime applies to the passing regime. However, after each overtake, we need to solve the differential equations again for the vehicles involved in passing and all the vehicles behind them.
Therefore, we need to compute the passing times or in other words the roots to the equations of type 
\begin{equation}
x_{i+1}\left(t\right)-x_{i}\left(t\right)=0
\end{equation}
or equivalently
\begin{equation}
z_{i+1}\left(t\right)-z_{i}\left(t\right)=0.
\end{equation}
The problem is to find the equation that has the smallest passing time and the passing time itself. This is necessary, so the coefficients in the solution can be corrected as soon as a passing occurs.

We have not developed any heuristics for the root finding algorithm, but it seems plausible that an algorithm can generate a short list of candidate equations that are suspected to have the smallest root based on various heuristics such as the distance between two vehicles and the velocity differences among other things. It is then easy to verify whether the obtained passing time is indeed minimal by checking that only one pass has occurred.

\subsection{Stability analysis}
As mentioned, a differential equation technique was used to turn 
Eq. \ref{EQ-variableChange1} and \ref{EQ-variableChange2} into the linear form of Eq. \ref{EQ-variableChange3}. However, since the variables $z_i$ in terms of which Eq. \ref{EQ-variableChange3} is linear are at their cores exponential functions, they can never be $0$. This is relevant since the point where all state variables are $0$ is the unique equilibrium point for linear systems of form $\frac{dq}{dt}=Aq$ where $det(A)\ne 0$. Putting Eq. \ref{EQ-variableChange3} in this matrix format will yield a lower triangular matrix $A$ whose diagonal elements are $\frac{-V_i}{\omega}$ and therefore non-zero. Since in a lower triangular matrix, the eigenvalues are the diagonal elements and no $0$ eigenvalue exists in this case, the determinant is non-zero. Therefore, we cannot use our analytical result for the purpose of stability analysis. In the next section, we rely on linearization and numerical simulation to study the stability of the model in the blocking regime.

\section{Model properties}
\subsection{Soundness}
In this section, we prove a few theorems that establish some of the expectations we have from a sound model.

We expect the velocities that are prescribed for each vehicle to be in the direction of the flow; that is non-negative. Here we show that our model never prescribes a negative velocity.
\begin{theorem}[Non-negative velocity]\label{Thm-Non-Negative-Velocity}
Given vehicle $i$ with maximum velocity $V_i$ in a platoon, $\frac{dx_i}{dt}\geq 0$.
\end{theorem}
\begin{proof}
For the sake of contradiction, and without loss of generality, let $n$ be the vehicle closest to the front in a platoon whose velocity will become negative. Call the moment when the velocity becomes zero, $t=t_0$. Taking a time derivative from both sides in Eq. \ref{eq-model1}, we will have
	\begin{equation}\label{Eq d^2x_i/dt^2}
		\frac{d^2x_i}{dt^2} = \frac{V_i}{\kappa\omega}\sum_{0\leq j<i}\left(\frac{dx_j}{dt}-\frac{dx_i}{dt}\right)\exp\left(\frac{x_i-x_j}{\omega}\right).
	\end{equation}
	Evaluating Eq. \ref{Eq d^2x_i/dt^2} at $t=t_0$, we will get
		\begin{equation}
		\frac{d^2x_i}{dt^2} = \frac{V_i}{\kappa\omega}\sum_{0\leq j<i}\left(\frac{dx_j}{dt}\right)\exp\left(\frac{x_i-x_j}{\omega}\right)>0
	\end{equation}
	where the strict inequality holds since no vehicle $j$ with $j<i$ can have a negative velocity due to our assumption and at least the first vehicle has a positive velocity. Since the derivative of velocity is positive, the velocity cannot become negative.
\end{proof}

Next, we prove that a vehicle with a smaller maximum velocity cannot pass a vehicle with a larger maximum velocity. One might perceive this is possible if the faster vehicle is subject to more congestion, but we show this will never be the case.

Before presenting the next theorem, let us first define a platoon. A platoon is referred to a group of vehicles that travel together while keeping their distances under some upper bound (i.e. the distance of the first to the last vehicle is always bounded by some constant). 

\begin{theorem}[No overtaking by slow]\label{Thm-No Overtake by slow}
Given vehicles $i$ and $i+1$ in a platoon, if $V_i\geq V_{i+1}$ and $x_i\left(t_0\right)>x_{i+1}\left(t_0\right)$, then $x_i\left(t\right)>x_{i+1}\left(t\right)$ for all $t\geq t_0$.
\end{theorem}
\begin{proof}Assume there exists some $t=t_p$, where $x=x_i\left(t_p\right)=x_{i+1}\left(t_p\right)$. At time $t=t_p$, we have $x_i=x_{i+1}$. By using Eq. \ref{eq-model1} for vehicle $i$, we can rewrite Eq. \ref{eq-model1} for vehicle $i+1$ as

\begin{equation*}
\frac{dx_{i+1}}{dt}=V_{i+1}\left(1-\left(\frac{1}{\kappa}+\Gamma_i\left(t\right)\right)\right)=
\end{equation*}

	 \begin{equation}\label{EQ dx_i+1/dt again}
	  	-\frac{V_{i+1}}{\kappa}+\frac{V_{i+1}}{V_{i}}\frac{dx_i}{dt}.
 	\end{equation}
	
	From Eq. \ref{EQ dx_i+1/dt again} above, it is clear that at $t=t_p$, $\frac{dx_i}{dt}> \frac{dx_{i+1}}{dt}$. This proves that passing will never be completed.

\end{proof}

\subsection{Passing or blocking behavior}\label{sec-passingCase}
First we prove a necessary and sufficient condition for a vehicle to pass another. Then we prove one of our main results that there exists a threshold for $\kappa$ above which, the model permits passing and below which it is not permitted. This constitutes a regime change in our model.

\begin{theorem}[Passing condition]\label{Thm-PassingCriteriaFor2}
Given vehicles $i$ and $i+1$ in a platoon with $V_{i+1}>V_i$ and $x_{i+1}\left(t_p\right)=x_i\left(t_p\right)$, vehicle $i+1$ will pass vehicle $i$ if and only if at the time of passing $t_p$ the following condition is met
\begin{equation}\label{Eq-Passing condition}
\kappa\left(1-\Gamma_i\left(t\right)\right)>\frac{V_{i+1}}{V_{i+1}-V_i}.
\end{equation}
\end{theorem}
\begin{proof}Eq. \ref{eq-model2} and 
theorem \ref{Thm-Non-Negative-Velocity} imply that $0\leq \Gamma_i\left(t\right)\leq 1$. We use Eq. \ref{eq-model1} for vehicle $i$ and Eq. \ref{EQ dx_i+1/dt again} for vehicle $i+1$ (which also holds here) in the following.
Vehicle $i+1$ will pass vehicle $i$ if and only if we  have (at time of passing)
\begin{equation*}
\frac{dx_{i+1}}{dt}-\frac{dx_i}{dt}>0 \Leftrightarrow
\end{equation*}
\begin{equation*}
\left(V_{i+1}-V_i\right)\left(1-\Gamma_i\left(t\right)\right)-\frac{V_{i+1}}{\kappa}>0 \Leftrightarrow
\end{equation*}
\begin{equation}
\kappa\left(1-\Gamma_i\left(t\right)\right)>\frac{V_{i+1}}{V_{i+1}-V_i}.
\end{equation}
\end{proof}

The next corollary states one of our main results.
\begin{corollary}[]\label{Thm-DW<=1}
$\kappa\leq 1$ is a sufficient condition for no passing to occur.
\end{corollary}
\begin{proof} Eq. \ref{eq-model2} and 
theorem \ref{Thm-Non-Negative-Velocity} imply that $0\leq \Gamma_i\left(t\right)\leq 1$. Since $0< \kappa\leq 1$, we have 
\begin{equation}
0\leq \kappa\left(1-\Gamma_i\left(t\right)\right)\leq 1.
\end{equation}
According to theorem \ref{Thm-PassingCriteriaFor2}, a faster vehicle $i+1$ will pass vehicle $i$ if and only if Eq. \ref{Eq-Passing condition} holds. But this will not hold as $1\leq\frac{V_{i+1}}{V_{i+1}-V_i}$. Therefore no passing occurs.
\end{proof}

\subsection{Asymptotic behavior in passing regime}

We cannot perform a straightforward stability analysis  in the passing regime since it is not clear how to conceptualize a reasonable equilibrium point in this case. 
However, the following theorems will be useful in understanding the passing regime in the asymptotic case.

\begin{theorem}[Order stability]\label{Thm-Order stability after T}
There exists a time $T$ after which the order of vehicles in the system will not change.
\end{theorem}
\begin{proof}
The number of possible orderings is fixed. Also, a slower vehicle cannot pass a faster vehicle according to theorem \ref{Thm-No Overtake by slow}. This creates a partial order on the set of ordering configurations. Therefore, at any state, either the system remains in that state forever or will move to a new state according to the partial order with no going back. Since the number of new admissible states is finite, the system will have to stay in one of the states forever after some time $T$.
\end{proof}

One might suspect that given enough time, vehicles will be sorted based on their maximum velocities; that is the fastest vehicle will become the first vehicle, the second fastest vehicle will be second, and so on. But as we will see in Theorem \ref{Thm-TimeInfinityOrder}, this will not necessarily be the case unless there is a meaningful difference between the velocities of any two vehicles. Intuitively, this can be understood in the following way; if a highway is congested to some extent and there are two vehicles that have slightly different maximum velocities, it is difficult for the fast vehicle to gain enough speed difference to take advantage of the little space available and overtake the slow vehicle.

\begin{theorem}[All fasts pass condition]\label{Thm-TimeInfinityOrder}
For a system with $N$ vehicles, as time goes to infinity, vehicles are guaranteed to be sorted via passing according to their maximum velocities, if and only if the following holds
\begin{equation}\label{Eq-GuaranteedSort}
\kappa>\max_{\substack{0\leq i,j\leq N-1 \\ i\neq j}}\left(\frac{V_j}{V_j-V_i}\right).
\end{equation} 
\end{theorem}
\begin{proof}
We first prove given the condition in Eq. \ref{Eq-GuaranteedSort}, a sorted order will be achieved.
From theorem \ref{Thm-Order stability after T}, the final order will be stable. We take this moment as the origin for time. For the sake of contradiction, assume the stable order is not sorted according to the maximum velocities. Let $i+1$ be the first vehicle with a larger maximum velocity than vehicle $i$, that is $V_{i+1}>V_{i}$. 
Since the vehicles in front of $i$ are faster than $i$, as the time goes to infinity, $\Gamma_{i}\left(t\right)$ goes to $0$. So for any given $\epsilon$, there exists some $t_\epsilon$ such that for $t>t_\epsilon\geq 0$ we have $\Gamma_{i}\left(t\right)<\epsilon$.
For any time $t>t_\epsilon$, We can rewrite Eq. \ref{eq-model1} for vehicle $i+1$ as
\begin{equation*}
\frac{dx_{i+1}}{dt}=
\end{equation*}
\begin{equation*}
    V_{i+1}\left(1-\left(\frac{1}{\kappa}\exp\left(\frac{x_{i+1}-x_i}{\omega}\right)+\Gamma_i\left(t\right)\right)\right)>
\end{equation*}
\begin{equation}\label{Eq-V_i+1 epsilon}
    V_{i+1}\left(1-\frac{1}{\kappa}-\epsilon\right).
\end{equation}
For vehicle $i$, we have $\frac{dx_i}{dt}\leq V_i$. To prove the passing occurs, it is sufficient to show $\frac{dx_{i+1}}{dt}\geq V_i+\epsilon'$ for all $t>t_\epsilon$ and some fixed $\epsilon'\geq 0$. Eq. \ref{Eq-GuaranteedSort} implies that

\begin{equation}\label{Eq-kappa epsilon}
\kappa>\frac{V_{i+1}}{V_{i+1}-V_{i}}\implies
\frac{1}{\kappa}=1-\frac{V_{i}}{V_{i+1}}-\epsilon''
\end{equation}
for some $\epsilon''>0$.
Replacing Eq. \ref{Eq-kappa epsilon} in Eq. \ref{Eq-V_i+1 epsilon}, yields
\begin{equation}
    \frac{dx_{i+1}}{dt}>V_i+V_{i+1}(\epsilon''-\epsilon)>V_i
\end{equation}
where the last inequality holds for any $\epsilon<\epsilon''$ and we pick one such $\epsilon$. Therefore $i+1$ will pass $i$ which will be a contradiction. Therefore, the order is only stable, if it is sorted according to the maximum velocities.

Now, to prove the other direction of the theorem, we show that for any value of $\kappa$, there exists at least a sequence of vehicles for which $0< \kappa\leq \max_{i,j}\left(\frac{V_j}{V_j-V_i}\right)$ holds and whose order is stable even though not sorted.
Let $M=\max_{i,j}\left(\frac{V_j}{V_j-V_i}\right)$. Without loss of generality, for any fixed $M$, our example only consists of two vehicles with velocity $V_0$ and $V_1$ where $V_1>V_0$ and $V_0$ and $V_1$ are chosen such that they satisfy $M=\frac{V_1}{V_1-V_0}$. Since we do not have $\kappa>\frac{V_j}{V_j-V_i}$, using lemma \ref{Thm-PassingThreshold} (see Appendix), vehicle $1$ will not pass vehicle $0$ and the order therefore will be stable but not sorted. 
\end{proof}

We summarize these results together with corollary  \ref{Thm-DW<=1} on the effect of $\kappa$ on how the model operates in Table \ref{table-DW}.

\begin{table}[!t]
\renewcommand{\arraystretch}{2.5}
\caption{Effect of capacity ($\kappa$) on the model's behavior}
\label{table-DW}
\centering
\begin{tabular}{|p{4cm}||p{3cm}|}
\hline
\textbf{Capacity ($\kappa$)} & \textbf{Model's behavior}\\
\hline
\textbf{Low:} \newline
$\kappa \leq 1$ & \textbf{Blocking regime}: \newline No vehicle can pass\\
\hline
\textbf{Medium:} \newline

$1<\kappa\leq\max_{i,j, i\ne j}\left(\frac{V_j}{V_j-V_i}\right)$ & \textbf{Passing regime}: \newline Initial position of vehicles determines the final ordering; that is which vehicles will end up passing\\
\hline
\textbf{High:} \newline
$\kappa>\max_{i,j,i\ne j}\left(\frac{V_j}{V_j-V_i}\right)$ &  \textbf{Passing regime}: \newline All faster vehicles end up ahead of slower ones\\
\hline
\end{tabular}
\end{table}

\begin{figure}[!t]
\centering
\includegraphics[trim=0 0 0 0, clip,width=3.5in]{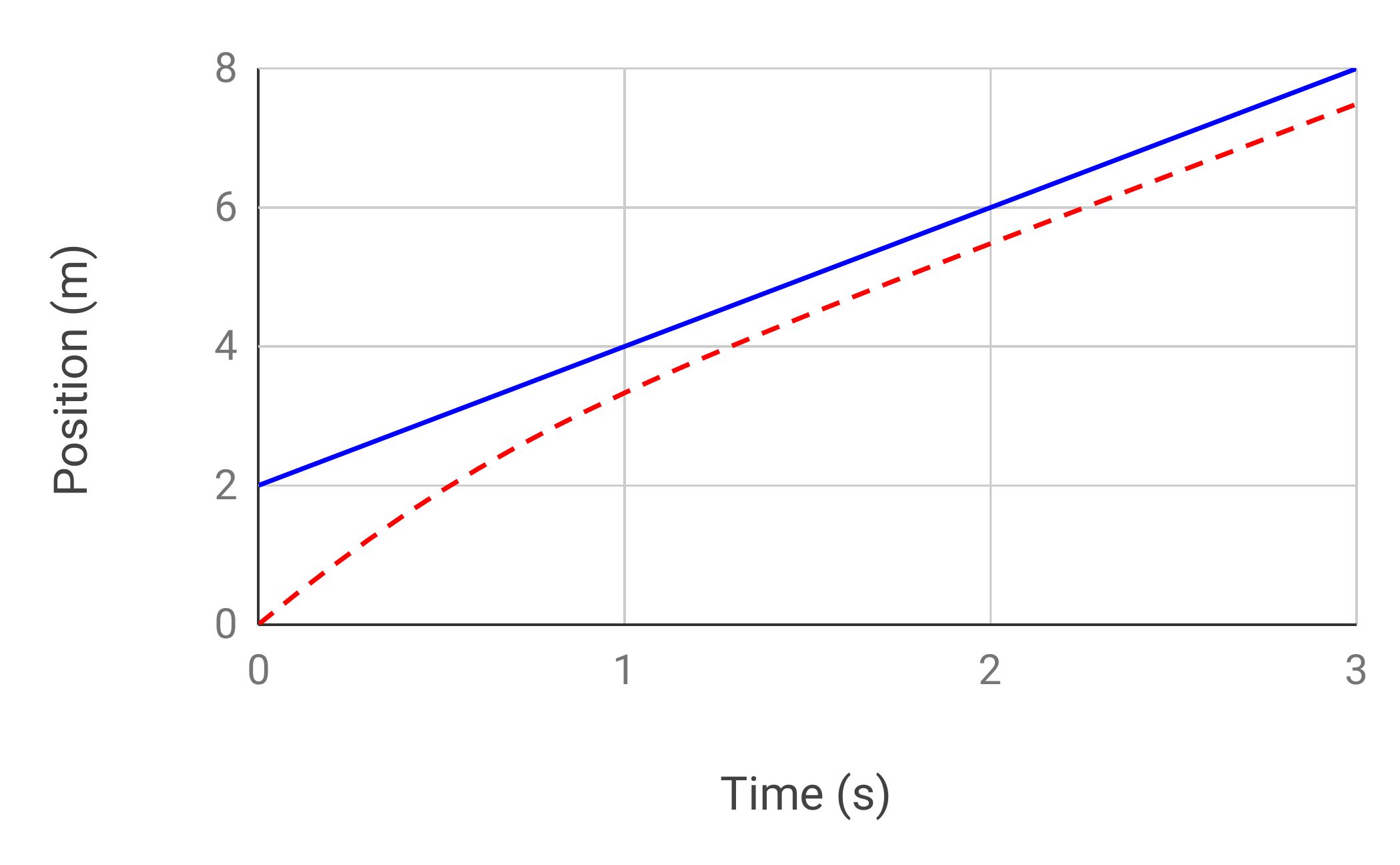}
\caption{In this case, due to the low capacity of the link ($\kappa = 1$) a faster vehicle gets stuck behind a slower vehicle.}
\label{fig:DW=1}
\end{figure}
\begin{figure}[!t]
\centering
\includegraphics[trim=0 0 0 0, clip,width=3.5in]{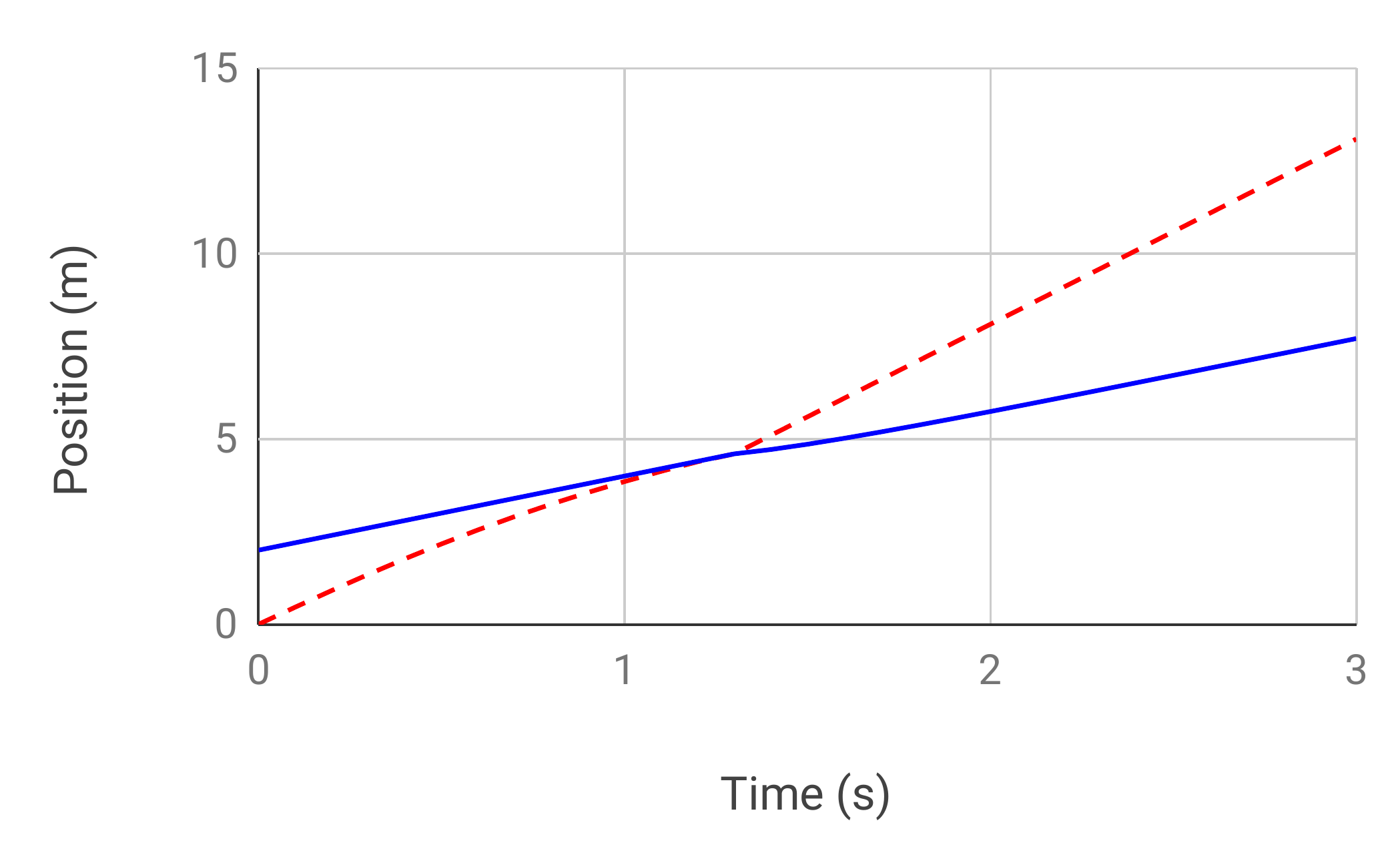}
\caption{In this case, the link has enough capacity ($\kappa = 2$) and a faster vehicle easily passes a slower vehicle.}
\label{fig:DW=2}
\end{figure}

\subsection{Asymptotic behavior in blocking regime: linear stability analysis}
The standard tool to study the asymptotic behavior in this case is stability analysis. We will study linear stability analysis for vehicles placed on an infinitely long road.
\subsubsection{Equilibrium point for the infinite road case}
Our state variables are the velocities of each vehicles excluding the first vehicle which has a constant velocity of $V_0$. Assuming we have a sequence of $N$ vehicles such that $V_i> V_0$ for all $i\leq N-1$, there exists an equilibrium point where all vehicles travel at the same velocity $v^{eq}=v_i^{eq}=V_0$.

\subsubsection{Local (platoon) and string stability analysis}
\begin{theorem}
In the blocking regime, given a platoon of $N$ vehicles, the Scalar Capacity Model has both local (platoon) and string stability. 
\end{theorem}
\begin{proof}
We first prove string stability of the model and then the local (platoon) stability will follow as a special case.
Our state variables are the velocities of all vehicles except the leader $V_0$ as it is fixed. Without loss of generality, we assume $V_i> V_0$ for all $i>0$ (otherwise, we do not have a single platoon according to lemma \ref{Thm-EqualSpeedSplitPlatoon} in the Appendix). We define the gap in front of vehicle $i$ and the gap between vehicle $n$ and $i$, respectively, as
\begin{equation}\label{EQ s_i}
s_{i} := x_{i-1}-x_{i},
\end{equation}
\begin{equation}\label{EQ s_n_i}
s_{n,i} := \sum_{i+1\leq j\leq n} s_j.
\end{equation}

In the equilibrium we have
\begin{equation}\label{Eq in equib}
v_{n}=v^{eq}=V_{n}\times
\left(1-\frac{1}{\kappa}\sum_{0\leq i< n}\exp\left(\frac{-s_{n,i}^{eq}}{\omega}\right)\right).
\end{equation}

Now we apply a small perturbation to the velocity of each follower as follows

\begin{equation}\label{Eq v_i 2}
v_i = v^{eq} + u_i\left(t\right),
\end{equation}
\begin{equation}\label{Eq s_i 2}
s_i = s^{eq} + y_i\left(t\right).
\end{equation}
From Eq. \ref{EQ s_i}, Eq. \ref{Eq v_i 2}, and Eq. \ref{Eq s_i 2}, we have
\begin{equation}\label{EQdyidtu}
\frac{dy_i}{dt}=-u_i\left(t\right).
\end{equation}
For $y_i$'s we define an identity similar to Eq. \ref{EQ s_n_i} as follows
\begin{equation}
y_{n,i} := \sum_{i+1\leq j\leq n} y_j.
\end{equation}
Assuming we kick all the follower vehicles out of equilibrium, for the $n$'th vehicle we will have

\begin{equation*}
v_{n}=v^{eq}+u_{n}\left(t\right)=V_{n}\times
\end{equation*}
\begin{equation}
\left(1-\frac{1}{\kappa}\sum_{0\leq i< n}\exp\left(\frac{-s_{n,i}^{eq}-y_{n,i}\left(t\right)}{\omega}\right)\right).
\end{equation}
After linearization and simplification using Eq. \ref{Eq in equib}, we get

\begin{equation*}
v^{eq}+u_{n}\left(t\right)=V_{n}\times 
\end{equation*}
\begin{equation*}
\left(1-\frac{1}{\kappa}\sum_{0\leq i< n}\exp\left(\frac{-s_{n,i}^{eq}}{\omega}\right)\left(1-\frac{y_{n,i}}{w}\right)\right)\implies
\end{equation*}
\begin{equation}\label{EQu_N-1String0}
u_{n}\left(t\right)=\frac{V_{n}}{\kappa\omega}\sum_{0\leq i< n}\exp\left(\frac{-s_{n,i}^{eq}}{\omega}\right)y_{n,i}.
\end{equation}
By replacing Eq. \ref{EQdyidtu} in Eq. \ref{EQu_N-1String0} and expanding $y_{n,i}$ according to its definition, we will get
\begin{equation}\label{EQu_N-1String}
\frac{dy_{n}}{dt}=\frac{-V_{n}}{\kappa\omega}\sum_{0\leq i<j\leq n}\exp\left(\frac{-s_{n,i}^{eq}}{\omega}\right) y_j.
\end{equation}
We can write Eq. \ref{EQu_N-1String} for $1\leq n\leq N-1$ for all vehicles in a matrix form as
\begin{equation}
\frac{dY}{dt} = AY.
\end{equation}
By inspection, $A$ is a lower triangular matrix with only negative elements. Since the eigenvalues of a lower triangular matrix are the elements of the diagonal, all the eigenvalues of the matrix are negative. Therefore, according to the linear stability theory, variables $y_i$ are stable with an equilibrium point of all $0$s. Hence, a similar thing can be said about $u_i$.
To understand the rate of convergence, we calculate the eigenvalues which are the elements of the diagonal of $A$. In other words, the eigenvalues $\lambda_{n}$ are the coefficients of $y_n$ in Eq. \ref{EQu_N-1String}. By inspection, we have
\begin{equation}\label{Eq lamba_n}
\lambda_{n}=\frac{-V_{n}}{\kappa\omega}\sum_{0\leq i< n}\exp\left(\frac{-s_{n,i}^{eq}}{\omega}\right)
=
\frac{V_0-V_n}{\omega}
\end{equation}
where the last equality is due to Eq. \ref{Eq in equib} and  knowing $v^{eq}=V_0$. Therefore, the bigger the difference between $V_0$ and $V_n$, the faster the convergence will be to the equilibrium point.
The above proved the string stability of the model. Linear (platoon) stability is proven by considering the special case where there is $0$ perturbation to the position and velocity of vehicles $1\leq i\leq N-2$ (that is $y_i=0$ and $u_i=0$)
\end{proof}
\subsection{Asymptotic behavior in blocking regime: nonlinear stability analysis}
In this section, We perform the non-linear stability analysis for vehicles placed on a ring road. Note that Eq. \ref{eq-model2} and Eq. \ref{eq-model3} are adjusted accordingly to become symmetrical for any vehicle $i$ (i.e. now each vehicle regardless of their numbering is a follower to every other vehicle on the ring road).
\subsubsection{Equilibrium point for the ring road}
Given a fleet of identical vehicles with maximum velocity $V_{max}$ travelling on a ring road, the exact locations of each vehicle is not important to us. However, their relative distance is important, we take the set of velocities $v_i$ 
as our state variables. Since the motion equations for all vehicles are symmetrical on the ring road, an immediately obvious equilibrium point is the case where all velocities are identical. This is equivalent to saying that all gaps are identical; that is \begin{equation}s^{eq}=s_i=\frac{L}{N}\end{equation} where $L$ is the circumference of the ring road and $N$ is the number of vehicles. We define the overall density $\rho$ as $\rho=\frac{1}{s^{eq}}$. The equilibrium velocity is calculated as follows:

\begin{equation*}
v^{eq} = V_{max}\left(1-\frac{1}{\kappa}\sum_{1\leq j\leq N-1}\exp\left(\frac{-js^{eq}}{\omega}\right)\right)=
\end{equation*}
\begin{equation}
V_{max}\left(1+\frac{1}{\kappa}-\frac{1}{\kappa}\sum_{0\leq j\leq N-1}\exp\left(\frac{-js^{eq}}{\omega}\right)\right).
\end{equation}
Using the identity for the sum of geometric series, we obtain

\begin{equation*}
v^{eq}=V_{max}\left(1+\frac{1}{\kappa}-\frac{1}{\kappa}\cdot\frac{1-\exp\left(\frac{-jNs^{eq}}{\omega}\right)}{1-\exp\left(\frac{-js^{eq}}{\omega}\right)}\right)=
\end{equation*}
\begin{equation}
v^{eq}=V_{max}\left(1+\frac{1}{\kappa}-\frac{1}{\kappa}\cdot\frac{1-\exp\left(\frac{-L}{\omega}\right)}{1-\exp\left(\frac{-1}{\rho\omega}\right)}\right).
\end{equation}
Vehicle flow, velocity, and density are related by $Q=v^{eq}\rho$ which results in the diagram in Fig. \ref{fig:fundamentalDiagram} relating the traffic flow $Q$ to the density $\rho$. This graph is also one of the fundamental diagrams of a traffic flow model and it gives insight into the macroscopic behavior of our microscopic model. It also gives an intuitive justification for the soundness of our model, since all traffic flow models (including those cited in this work) produce a more or less similar graph. That is the traffic flow increases as the density increases till we reach a peak capacity after which adding any more vehicles will only serve to decrease the flow.

\begin{figure}[!t]
\centering
\includegraphics[trim=0 0 0 0, clip,width=3.5in]{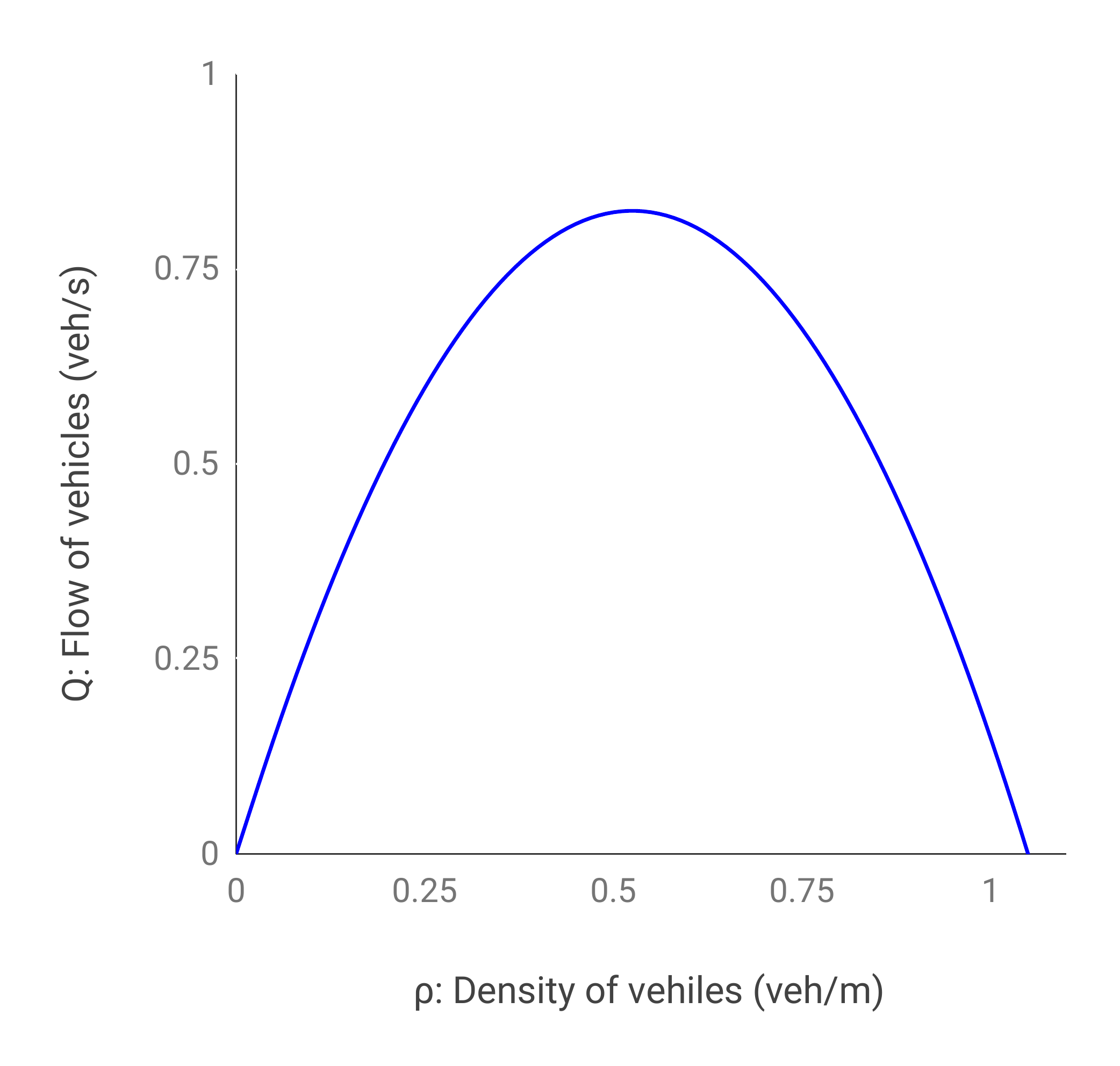}
\caption{Vehicle flow versus density: For our microscopic model, this graph shows the macroscopic relationship between the number of identical vehicles passing a fixed point on a ring road per unit of time and the density of vehicles. As is expected from a traffic flow model, after a peak density matching the available capacity is reached, traffic flow starts deteriorating in the sense that any more vehicles only serves to slow down every vehicle. Before this peak, the traffic is in the free flow regime and then switches to congested.}
\label{fig:fundamentalDiagram}
\end{figure}

\subsubsection{Numerical experiments}
In this section, we initiate the system with a variety of conditions and observe whether the system will approach to the equilibrium point.
We will use a chosen background density composed with a smaller region of higher density. We will observe how this irregularity will affect the system's stability. Our experiments parameters are chosen as follows:
\begin{itemize}
\item $L = 1000m$, length of the ring road ($m$)
\item $\omega = 10m$, length of the horizon in front of each vehicle.
\item $V_{max} = 6m/s$, for all vehicles
\item $\kappa = 10$, model's capacity
\item $t_{start}=0s, t_{end}=500s$, start and finish time of simulation
\item $\rho=0.5\frac{veh}{m}$, global density of vehicles. In other words, we have $500$ vehicles on the ring road (a minimum number of vehicles that is required for a realistic simulation \cite{Tre13}).
\end{itemize}
To produce Fig. \ref{TimeSpace30} and \ref{TimeGap30} we distribute the vehicles in two regions. One region consists of 30\% of the ring road and has the highest possible uniform density of the vehicles and the remaining vehicles are distributed in the rest of the ring road evenly; so to make the overall density $\rho$ as above. These experiments provide evidence that no matter how far from equilibrium the system is, it will converge to the equilibrium. Also, we performed the same test with the same number of vehicles when 10\% or 20\% of the ring road had a maximal traffic jam and each case produced essentially the same graphs.

\begin{figure}[!t]
\centering
\includegraphics[trim=0 0 0 0, clip,width=3.5in]{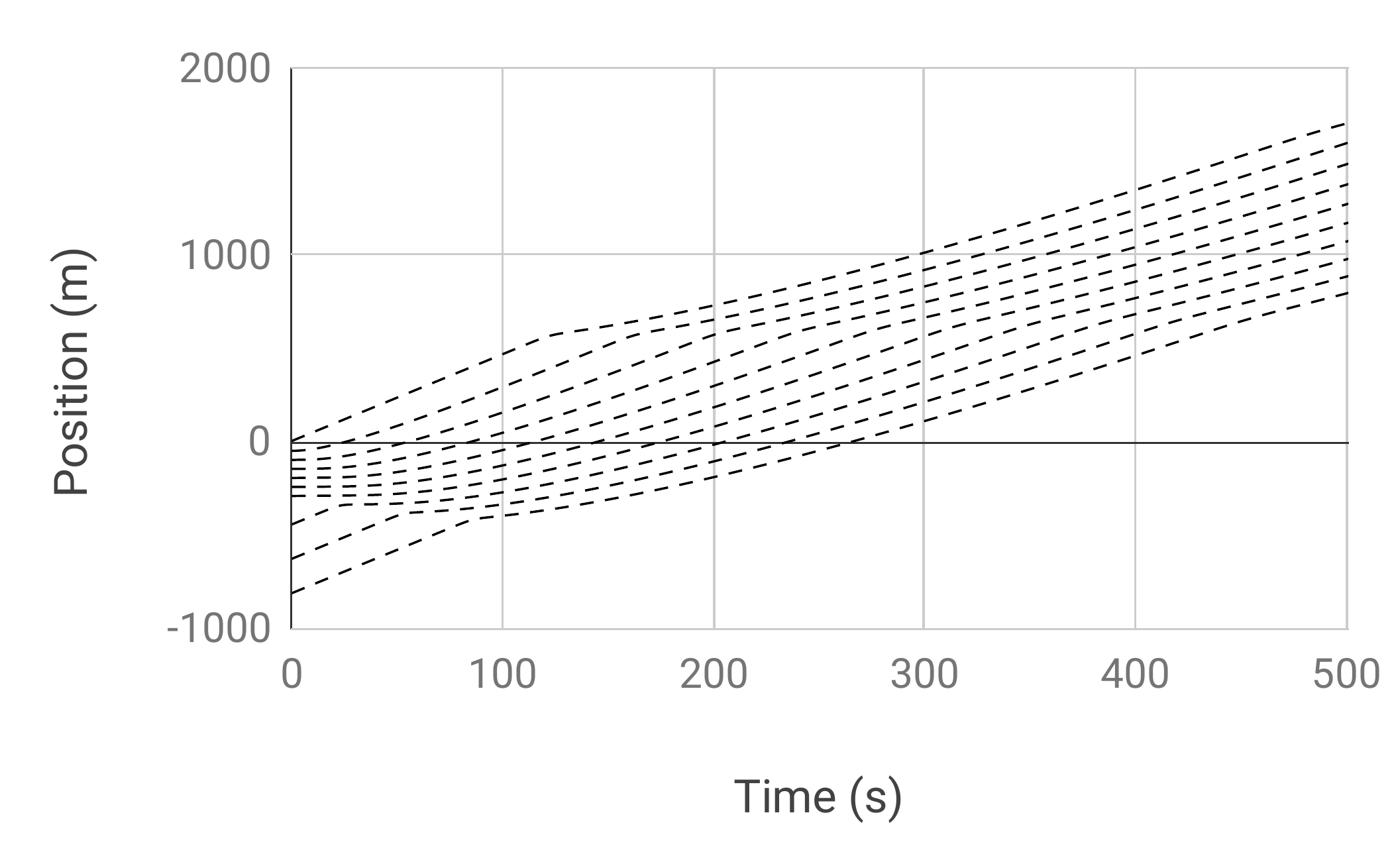}
\caption{Time-Space Diagram for every 50'th vehicles: $30\%$ of the ring road has a maximal uniform vehicle density of $1.03\frac{veh}{m}$ and the remaining $70\%$ has a uniform density of $0.27\frac{veh}{m}$. The ring road vehicle density is $0.5\frac{veh}{m}$. Initially, some vehicles are slowed down, but as time goes on, all the velocities converge to the equilibrium velocity.}
\label{TimeSpace30}
\end{figure}

\begin{figure}[!t]
\centering
\includegraphics[trim=0 0 0 0, clip,width=3.5in]{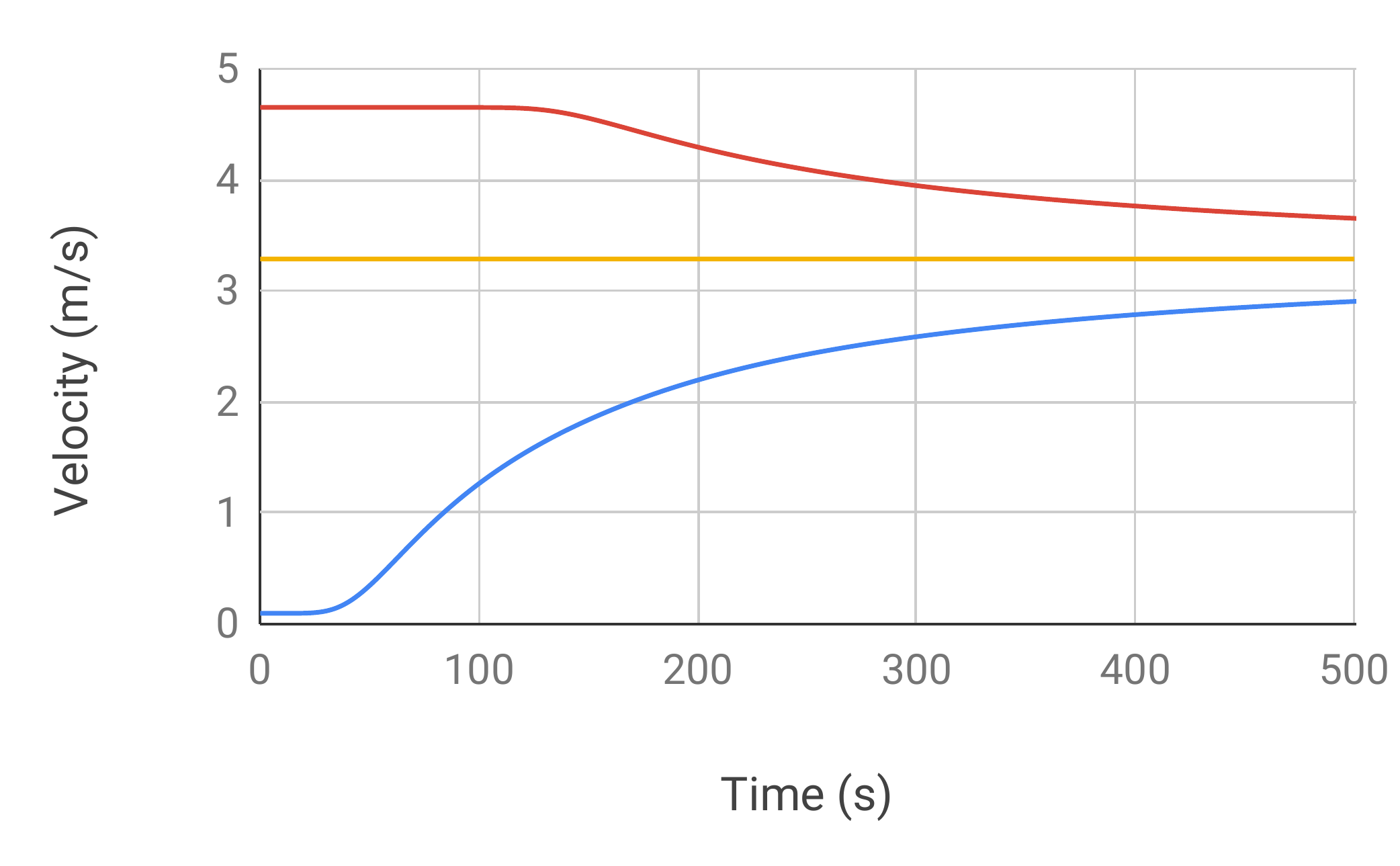}
\caption{Minimum and maximum momentary velocity among all vehicles: $30\%$ of the ring road has a maximal uniform vehicle density of $1.03\frac{veh}{m}$ and the remaining $70\%$ has a uniform density of $0.27\frac{veh}{m}$. The ring road vehicle density is $0.5\frac{veh}{m}$. This graph shows the convergence to equilibrium velocity.}
\label{TimeGap30}
\end{figure}

\section{Discussion \& Future work}
Our work leaves many open questions. An important question is whether it is possible to add some mechanism for delay, without losing the closed form analytical solution feature of the model. 

Another avenue for research is adding some dummy vehicles to play the role of moving or stationary obstacles for the traffic flow. This works by consuming the capacity of the airway/road dynamically. It organically gives rise to inclusion of obstacles without modifying the model. In the same vein, it is possible to add weights to the exponential congestion term for different vehicles or dummy vehicles (obstacles). Currently in our model, all these weights are  equal. The ability to set weights can give us powerful tools for tuning the strength of these obstacles.

\section{Conclusion}
In this paper, we introduced a microscopic traffic flow model that can be used to study traffic patterns of unmanned aerial vehicles in the air as they become ubiquitous in the future. The model is equally applicable to the study of the traffic flow of ground vehicles on the road. We advanced the state of art by introducing a scalar capacity parameter for the airway (or roads) rather than the traditional approach of modelling links as 1 lane or multi-lane. This is suited for the study of the 3D nature of UAV flights as opposed to the 2D nature of ground vehicles movements while also resulting in a simpler model for ground vehicles by abstracting away the pass planning aspect. By adjusting the scalar capacity parameter, the model can exhibit passing or blocking behaviors. In the former, vehicles are free to pass each other while in the latter, no vehicle can pass another one similar to a one lane road. Our model can be solved analytically for the blocking regime and piece-wise analytically in the passing regime. For the blocking regime we proved linear local (platoon) stability as well as string stability. Also, using numerical simulation, we show evidence for non-linear stability. For the passing regime, we proved theorems outlining the asymptotic behavior of the model such as whether every faster vehicle gets a chance to pass slower vehicles as time goes to infinity and what the final order of vehicles will be after all the overtakings are completed. Lastly, we proved a main theorem characterizing the transition from blocking to passing as we adjust the scalar capacity parameter.

\appendix  

\begin{lemma}[Passing threshold]\label{Thm-PassingThreshold}
Given only two vehicles on the road with $V_1>V_0$, given enough time, vehicle $1$ will pass vehicle $0$  if and only if $\kappa>\frac{V_1}{V_1-V_0}$.
\end{lemma}
\begin{proof} Proving vehicle $1$ passes vehicle $0$ implies $\kappa>\frac{V_1}{V_1-V_0}$ is a straightforward consequence of theorem \ref{Thm-PassingCriteriaFor2} since $\Gamma_0 = 0$ at all times including the passing time between vehicles $0$ and $1$. The velocity for vehicle $1$ at the passing time is in fact the smallest value for the follower's velocity (that is any starting gap will close eventually, if passing can be done if we magically put both vehicles in the same position).

In the other direction, we prove $\kappa > \frac{V_1}{V_1-V_0}$ implies for all times, that $v_1\left(t\right)>v_0\left(t\right)$. We first prove the two vehicles will meet as a condition for theorem \ref{Thm-PassingCriteriaFor2}, so we can apply that theorem.

$v_1$ takes its minimum value $v_{1_{min}}$ when vehicle $0$ and $1$ are (hypothetically) in the same position, and according to theorem \ref{Thm-PassingCriteriaFor2}, a pass will occur in that case since $\kappa > \frac{V_1}{V_1-V_0}$. Therefore, at all other times, $v_1\geq v_{1_{min}}>v_0\left(t\right) = V_0$ and that implies that there exists a time $t_p$ when the two vehicles will meet. Therefore, according to theorem \ref{Thm-PassingCriteriaFor2}, vehicle $0$ will pass vehicle $1$.
\end{proof}

Without loss of generality, assume vehicle $n$ is the first vehicle for which $V_n=V_0$. In the following lemma, we show if a vehicle $n$ in a sequence of vehicles following a leader with speed $V_0$ has a maximum speed $V_n=V_0$, then this will result in creation of two platoons. When $V_n<V_0$, this is easy to see. But when the maximum speeds are equal, one can see that this still holds. More formally:
\begin{lemma}[Platoon splitting]\label{Thm-EqualSpeedSplitPlatoon}
Given a platoon of $n$ vehicles in the blocking regime with an extra vehicle $n$ with $V_n =V_0$ and where vehicle $0$ to $n-1$ are in equilibrium, if $V_0<V_j$ for $1\leq j\leq n-1$, then vehicle $n$ is not part of the  platoon.
\end{lemma}
\begin{proof}[Proof of lemma \ref{Thm-EqualSpeedSplitPlatoon}]

From Eq. \ref{eq z_i solution} and knowing $V_0<V_j$ for $1\leq j\leq n-1$ and $V_0=V_n$, we have

\begin{equation*}
z_{n} = \left(c_{n,0,0}+c_{n,0,1}\cdot t\right)\exp\left(\frac{-V_0t}{\omega}\right)+
\end{equation*}
\begin{equation}\label{eq z_i solution special case}
    \sum_{\substack{j\in U-\{0\} \\ j\leq n}}\sum_{0\leq d<m_{n,j}}c_{n,j,d}\cdot t^d\exp\left(\frac{-V_jt}{\omega}\right).
\end{equation}
From the definition of $z_0$ in Eq. \ref{eq def z_i} we have 
\begin{equation}\label{eq z0 234}
    z_{0} = \exp\left(\frac{-x_0}{\omega}\right)=
    \exp\left(\frac{-x_0(0)-V_0t}{\omega}\right).
\end{equation}
Now, by using Eq. \ref{eq z_i solution special case} and Eq. \ref{eq z0 234}, we get
\begin{equation*}
    \lim_{t\rightarrow \infty} 
    \exp\left(\frac{x_0-x_n}{\omega}\right)=
    \lim_{t\rightarrow \infty} 
    \frac{z_n}{z_0}=\infty\implies
\end{equation*}
\begin{equation}
    \lim_{t\rightarrow \infty} \left(x_0-x_n \right)\rightarrow \infty.
\end{equation}
Therefore vehicle $n$ cannot be part of the same platoon of vehicles $0$ to $n-1$.
\end{proof}










%



\section*{Acknowledgment}
We would like to thank Claudio Ca\~{n}izares, Ehsan Nasr Azadani, and Amir Khajepour for fruitful discussions on the subject.

\ifCLASSOPTIONcaptionsoff
  \newpage
\fi



\bibliographystyle{IEEEtran}
\bibliography{IEEEabrv,refs}
%





%

\begin{IEEEbiography}[{\includegraphics[width=1in,height=1.25in,clip,keepaspectratio]{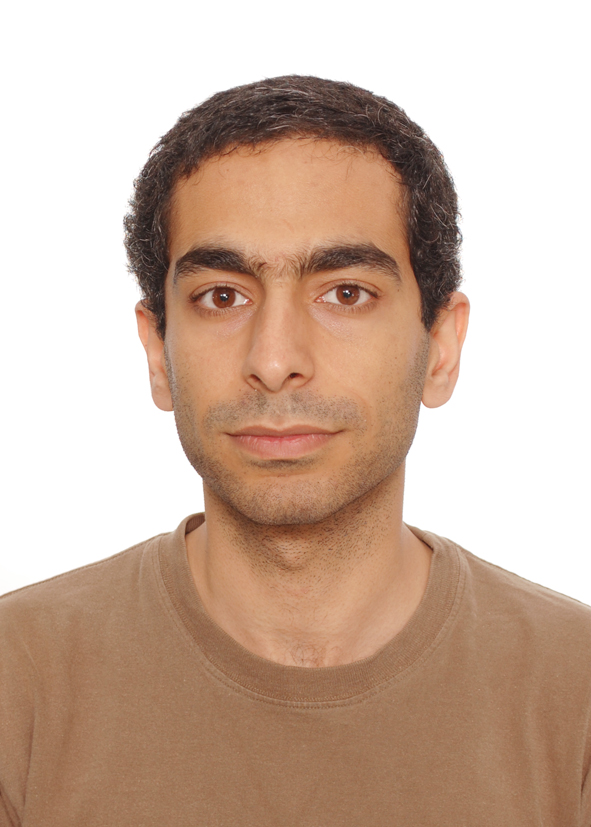}}]{Mirmojtaba Gharibi}
received the B.ASc. in electrical engineering from Sharif University of Technology in 2009. He completed his M.Math, and currently is pursuing a PhD degree both in Computer Science at University of Waterloo, Waterloo, Canada.
\end{IEEEbiography}

\begin{IEEEbiography}[{\includegraphics[width=1in,height=1.25in,clip,keepaspectratio]{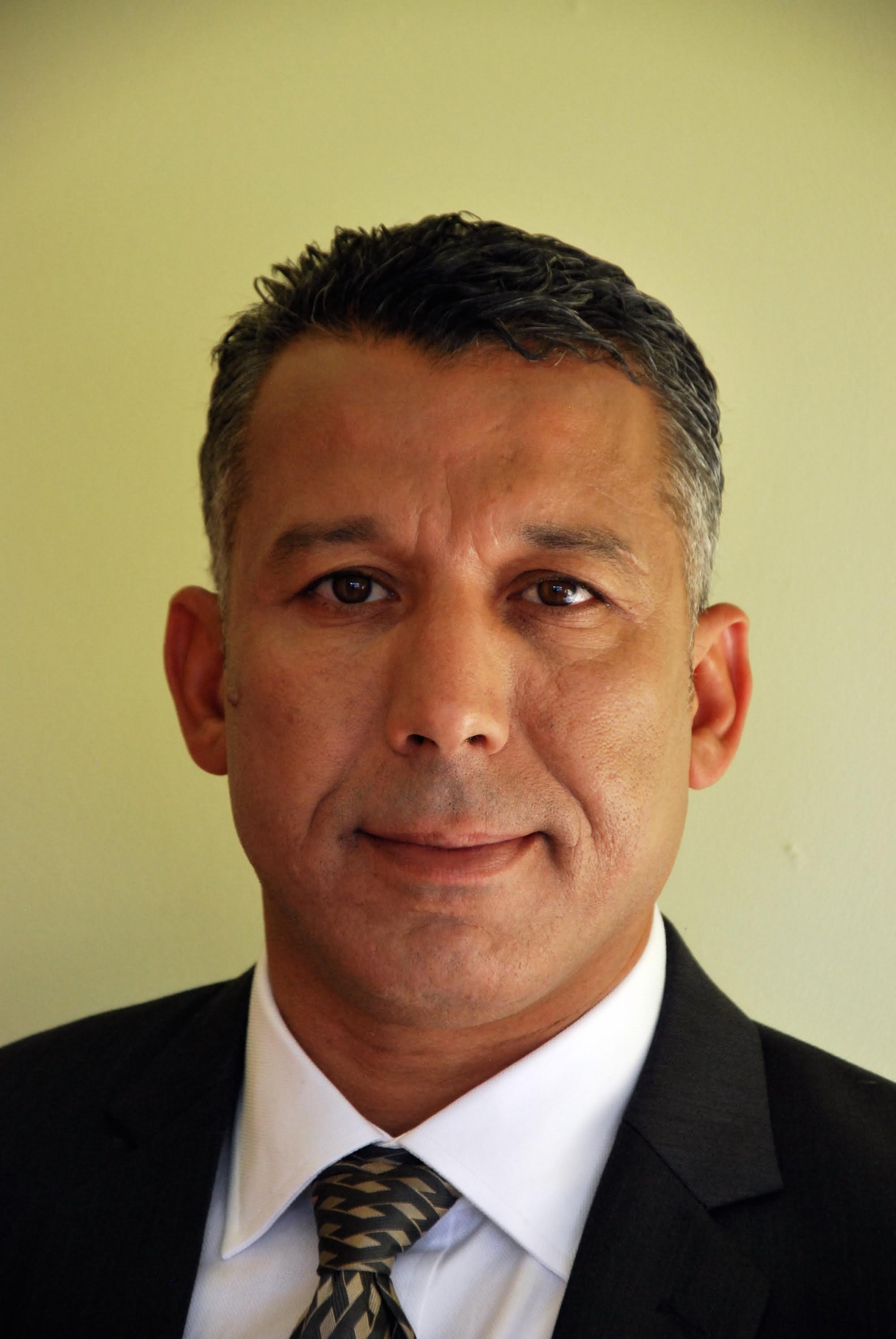}}]{Raouf Boutaba}
received the M.Sc. and Ph.D.
degrees in computer science from the University
Pierre \& Marie Curie, Paris, in 1990 and 1994,
respectively.
He is currently a university chair, a professor of computer science and the Associate Dean Innovation \& Entrepreneurship of the
faculty of Mathematics at the University of Waterloo (Canada). His research interests include
resource management in wired and wireless networks. He was the founding editor in chief of the IEEE transactions on Network and Service Management (2007-2010) and the current editor in chief of the IEEE journal on selected
areas in communications. He is a fellow of the IEEE, the Engineering
Institute of Canada, and the Canadian Academy of Engineering.
\end{IEEEbiography}

\begin{IEEEbiography}[{\includegraphics[width=1in,height=1.25in,clip,keepaspectratio]{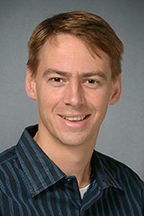}}]{Steven L. Waslander} received the B.Sc.E.
degree from Queen's University, Kingston, ON,
Canada, in 1998, and the M.S. and Ph.D. degrees
from Stanford University in Aeronautics and Astronautics, Stanford, CA, USA, in 2002 and 2007,
respectively.
As a graduate student at Stanford University
in Aeronautics and Astronautics, he created the
Stanford Testbed of Autonomous Rotorcraft for
Multi-Agent Control, the world’s most capable outdoor multi-vehicle quadrotor platform at the time.
He was a Control Systems Analyst for Pratt \& Whitney Canada, Longueuil,
QC, Canada, from 1998 to 2001. He was recruited to the University of
Waterloo, Waterloo, ON, Canada, in 2008, and moved to the University of
Toronto, Toronto, ON, Canada, in 2018, where he founded and directs the
Toronto Robotics and Artificial Intelligence Laboratory. He is currently an
Associate Professor with the University of Toronto Institute for Aerospace
Studies, Toronto, where he is a Leading Authority on autonomous aerial
and ground vehicles, simultaneous localization and mapping, and multivehicle systems. His current research interests include the state of the art in
autonomous drones and autonomous driving through advances in localization
and mapping, object detection and tracking, integrated planning and control
methods, and multirobot coordination.

\end{IEEEbiography}




\end{document}